\documentclass{article}

\usepackage{microtype}
\usepackage{graphicx}
\usepackage{subfigure}
\usepackage{booktabs} 

\usepackage{amsmath,amssymb}
\usepackage{algorithm} 
\usepackage{algorithmic}
\usepackage{bm}
\usepackage{array}
\usepackage{makecell}
\usepackage{amsmath}
\newtheorem{theorem}{Theorem}
\newtheorem{proof}{Proof}

\usepackage{hyperref}

\usepackage[accepted]{icml2020}

\icmltitlerunning{PDO-eConvs: Partial Differential Operator Based Equivariant Convolutions}

\begin{document}

\twocolumn[
\icmltitle{PDO-eConvs: Partial Differential Operator Based Equivariant Convolutions}




\begin{icmlauthorlist}
	\icmlauthor{Zhengyang Shen}{math}
	\icmlauthor{Lingshen He}{info}
	\icmlauthor{Zhouchen Lin}{info}
	\icmlauthor{Jinwen Ma}{math}
\end{icmlauthorlist}

\icmlaffiliation{math}{School of Mathematical Sciences and LMAM, Peking University, Beijing 100871}
\icmlaffiliation{info}{Key Lab. of Machine Perception (MoE), School of EECS, Peking University, Beijing 100871}

\icmlcorrespondingauthor{Zhouchen Lin}{zlin@pku.edu.cn}
\icmlcorrespondingauthor{Jinwen Ma}{jwma@math.pku.edu.cn}

\icmlkeywords{Machine Learning, ICML}

\vskip 0.3in
]



\printAffiliationsAndNotice{}  

\begin{abstract}
Recent research has shown that incorporating equivariance into neural network architectures is very helpful, and there have been some works investigating the equivariance of networks under group actions. However, as digital images and feature maps are on the discrete meshgrid, corresponding equivariance-preserving transformation groups are very limited. 

In this work, we deal with this issue from the connection between convolutions and partial differential operators (PDOs). In theory, assuming inputs to be smooth, we transform PDOs and propose a system which is equivariant to a much more general continuous group, the $n$-dimension Euclidean group. In implementation, we discretize the system using the numerical schemes of PDOs, deriving approximately equivariant convolutions (PDO-eConvs). Theoretically, the approximation error of PDO-eConvs is of the quadratic order. It is the first time that the error analysis is provided when the equivariance is approximate. Extensive experiments on rotated MNIST and natural image classification show that PDO-eConvs perform competitively yet use parameters much more efficiently. Particularly, compared with Wide ResNets, our methods result in better results using only $12.6\%$ parameters.
\end{abstract}

\section{Introduction \label{intro}}
In the past few years, convolutional neural network (CNN) models have become the dominant machine learning methods in the field of computer vision for various tasks, such as image recognition, objective detection and semantic segmentation. Compared with fully-connected neural networks, a significant advantage of CNNs is that they are shift equivariant: shifting an image and then feeding it through a number of layers is the same as feeding the original image and then shifting the resulted feature maps. In other words, the translation symmetry is preserved by each layer. Also, the equivariance property brings in weight sharing, with which we can use parameters more efficiently.

Motivated by this, Cohen and Welling \yrcite{cohen2016group} proposed group equivariant CNNs (G-CNNs), showing how convolutional networks can be generalized to exploit larger groups of symmetries, including rotations and reflections. G-CNNs are equivariant to the group $p4m$ or $p4$\footnote{Generally, the group $pnm$, which we will use in Section \ref{section4}, denotes the group generated by translations, reflections and rotations by $2\pi/n$. The group $pn$ denotes the group only generated by translations and rotations by $2\pi/n$.}, and work on square lattices. In addition, Hoogeboom et al. \yrcite{hoogeboom2018hexaconv} proposed HexaConv and showed how one can implement planar convolutions and group convolutions over hexagonal lattices, instead of square ones. As a result, the equivariance is expanded to $p6m$. However, it seems impossible to design CNNs that are equivariant to the rotation angles other than $\pi/2$ ($p4m$) and $\pi/3$ ($p6m$) as there does not seem to exist other rotational symmetric discrete lattices on the 2D plane, if one considers equivariance in the ways as \cite{cohen2016group} and \cite{hoogeboom2018hexaconv}.

In order to exploit more symmetries, Weiler et al. \yrcite{weiler2018learning} employed harmonics as steerable filters to achieve exact equivariance to larger transformation groups in the continuous domain. However, they are difficult to preserve strong equivariance when operating on discrete pixel grids, for two main reasons: (i) When a harmonic is sampled on grids with a low rate, it could appear as a lower harmonic, which introduces aliasing artifacts. (ii) With Gaussian radial profiles as radial functions, harmonics ranged out of the sampled kernel support, leading to a high equivariance error on implementation.

From another point of view, a conventional convolutional filter can also be viewed as a linear combination of PDOs, which was proposed by \cite{ruthotto2018deep}. With this new understanding, we assume inputs are smooth functions, and then show how to transform the PDOs and get a system which is exactly equivariant to a much more general continuous transformation group, the $n$-dimension Euclidean group. To implement our theory on discrete digital images, we discretize the system using the numerical schemes of PDOs and get approximately equivariant convolutions. Particularly, the discretized convolutions can achieve a quadratic order equivariance approximation, and it is the first time that the error analysis is provided when the equivariance is approximate. As the derived equivariant convolutions are based on PDOs, we refer to them as PDO-eConvs.

We evaluate the performance of PDO-eConvs on rotated MNIST and natural image classification tasks. Extensive experiments verify that PDO-eConv produces very competitive results and is significantly efficient on parameter learning.. 

Our contributions are as follows:

\begin{itemize}
	\item With the assumption that inputs are smooth, we use PDOs to design a system that is equivariant to a much more general continuous group, the $n$-dimensional Euclidean group.
	\item The equivariance is exact in the continuous domain. It becomes approximate only after the discretization. Moreover, it is the first time that the error analysis is provided when the equivariance is approximate. To be specific, the approximation error of PDO-eConvs is of the quadratic order, indicating a precise approximation.
	\item Extensive experiments on PDO-eConvs show that our methods perform competitively and have significant parameter efficiency. 
\end{itemize}

\section{Prior and Related Work}
\subsection{Equivariant CNNs}

Lenc \& Vedaldi \yrcite{lenc2015understanding} showed that the AlexNet CNN \cite{krizhevsky2012imagenet} trained on ImageNet spontaneously learned representations that are equivariant to flips, scalings and rotations, which supported the idea that equivariance is a good inductive bias for CNNs. Cohen \& Welling \yrcite{cohen2016group,cohen2017steerable} succeeded in incorporating equivariance into neural networks. However, these methods can only deal with a $4$-fold rotational symmetry for images with square pixels. Hoogeboom et al. \yrcite{hoogeboom2018hexaconv} alleviated this limit by implementing planar convolutions and group convolutions over hexagonal lattices. Consequently, they can deal with a $6$-fold rotational symmetry.

Since there does not seem to have more rotational symmetries on lattices in the 2D plane, some works designed approximately equivariant networks w.r.t. larger groups. Zhou et al. \yrcite{zhou2017oriented} and Marcos et al. \yrcite{marcos2017rotation} utilized bilinear interpolation to help produce feature maps at different orientations. They are inherently approximately equivariant. By comparison, ours is exactly equivariant in the continuous domain. Worral et al. \yrcite{worrall2017harmonic} used harmonics to extract features and achieve equivariance to 360-rotation, but the equivariance is destroyed after Gaussion-resampling. Weiler et al. \yrcite{weiler2018learning} and Weiler \& Cesa  \yrcite{cesa2019general} employed harmonics as steerable filters to achieve exact equivariance w.r.t. larger groups in the continuous domain, but the equivariance is difficult to preserve in the discrete domain due to aliasing artifacts and limited kernel support. So they used much larger filters to achieve approximate equivariance, resulting in CNNs with a large computational burden. By contrast, PDO-eConvs can use a relatively small kernel size to achieve theoretically guaranteed exact equivariance in the discrete domain, which makes big difference. 


There are also some empirical approaches for enforcing equivariance. A commonly utilized technique is data augmentation, see e.g. \cite{krizhevsky2012imagenet}. The basic idea is to enrich the training set by transformed samples. Laptev et al. \yrcite{laptev2016ti} used parallel siamese architectures for the considered transformation set and applying the transformation-invariant pooling (TI-Pooling) operator on their outputs. Jaderberg et al. \yrcite{jaderberg2015spatial} applied a differentiable module to actively transform feature maps, and then Esteves et al. \yrcite{esteves2018polar} used this method to help enforce equivariance under rotation and scale transformations. In \cite{sabour2017dynamic,hinton2018matrix}, capsules are used to represent the location information and enforce equivariance. However, these methods learn the transformations directly from datas, which are inferior to those methods incorporating equivariance into architectures for lack of interpretability and reliability.

\subsection{The Relationship between Convolutions and PDOs}

There have been extensive works \cite{jain1978partial,witkin1987scale-space,koenderink1984the,perona1990scale-space,osher1990feature-oriented} utilizing PDOs to process images. The relationship between convolutions and PDOs was presented in \cite{dong2017image,ruthotto2018deep}, where the authors translated convolutional filters to linear combinations of PDOs, and this approximation has good analytical properties. Some works \cite{long2018pde,long2019pde} used this new understanding to help design CNNs. Also, this relationship is an important theoretical foundation of our work.

Actually, there exist some works using PDOs to investigate equivariance. Liu et al. \yrcite{liu2013toward} designed a partial differential equation (PDE) using a linear combination of equivariant PDOs and proposed learning based PDEs, which are naturally shift and rotation equivariant. Fang et al. \yrcite{fang2017feature} further adopted this technique on face recognition task. However, the capacity of learning based PDEs cannot be compared with that of nowadays widely used CNNs.

\section{Mathematical Framework}
In this section we design a group equivariant system using PDOs. To make concepts and notations more explicit, we give a preliminary introduction of groups and equivariance formally. 

\subsection{Prior Knowledge}
\textbf{The Isometry Group}\quad In mathematics, the isometry group is a group consisted of isometry transformations, which preserve the distance of any two points. Particularly, the Euclidean group is the largest isometry group defined on $\mathbb{R}^n$, which we denote as $E(n)$. Given $y\in \mathbb{R}^n$, the isometry transformation is:
\begin{align}
y:\to Ay+x,
\label{11}
\end{align}
where $A$ is an orthogonal matrix, i.e., $A^{\top}A=I$, and $x \in\mathbb{R}^n$. When $A=I$, the transformations in (\ref{11}) compose the translation group $\langle \mathbb{R}^n;+\rangle$. Without ambiguity, we use $\mathbb{R}^n$ to denote the translation group in the following text. When $x=0$, $E(n)$ degenerates to the orthogonal group, $O(n)$, which contains all the orthogonal transformations, including reflections and rotations. We use $A$ to parameterize $O(n)$. $\mathbb{R}^n$ and $O(n)$ are both subgroups of $E(n)$, 
and $E(n)=\mathbb{R}^n \rtimes O(n)$  ($\rtimes$ is a semidirect-product). We use $(x,A)$ to represent the element in $E(n)$, where $x$ and $A$ represent a translation and an orthogonal transformation, respectively. Restricting the domain of $A$ and $x$, we can also use this representation to parametrize any subgroup of $E(n)$. 

\textbf{Actions on Functions}\quad Inputs and intermediate feature maps can be naturally modeled as functions defined in the continuous domain. To be specific, we model the input $r$ as a smooth function defined on $\mathbb{R}^n$ and the intermediate feature map $e$ as a smooth function defined on $E(n)$, where the smoothness of $e$ means that if we use the representation $(x,A)$ mentioned above, the feature map $e(x,A)$ is smooth w.r.t. $x$ when $A$ is fixed. So $e$ can also be viewed as a function defined on $\mathbb{R}^n$ with infinite channels indexed by $A$. We use $C^{\infty}(\mathbb{R}^n)$ and $C^{\infty}(E(n))$\footnote{For the simplicity of our theory, we require that $r\in C^{\infty}(\mathbb{R}^n)$. However, in implementation, we only require that $r\in C^{4}(\mathbb{R}^n)$. The requirement on $e$ is the same.} to denote the function spaces of $r$ and $e$, respectively .

In this way, transformations like rotations and reflections on inputs and feature maps can be mathematically formulated. Here, we introduce two transformations used in our theory.
\begin{itemize}
	\item Suppose that $r\in C^\infty(\mathbb{R}^n)$ and $\widetilde{A} \in O(n)$, then the transfomation $\widetilde{A}$ acts on $r$ in the following way\footnote{We use $[\cdot]$ to denote that an operator acts on a function.}: 
	\begin{align}
	\forall x \in \mathbb{R}^n,\quad \pi^{R}_{\widetilde A}[r](x)=r(\widetilde A^{-1}x).
	\label{T1}
	\end{align}
	
	\item Suppose that $e \in C^\infty(E(n))$ and $\widetilde A\in O(n)$, then $\widetilde A$ acts on $e$ in the following way: 	
	\begin{align}
	\forall a \in E(n),\quad \pi^{E}_{\widetilde A}[e](a)=e({\widetilde{A}}^{-1}a),
	\label{31}
	\end{align}
	where $\widetilde{A}^{-1}a$ is group product on $E(n)$. Using the representation of $E(n)$, it is of the
	following more detailed form:
	\begin{align}
	\pi^{E}_{\widetilde A}[e](x,A)=e( \widetilde {A}^{-1}x, \widetilde{A}^{-1} A), \label{T2}
	\end{align} 
	where $(x,A)$ is the representation of $a$.
	
\end{itemize}

\begin{figure}
	\centering
	\includegraphics[width=0.6\columnwidth]{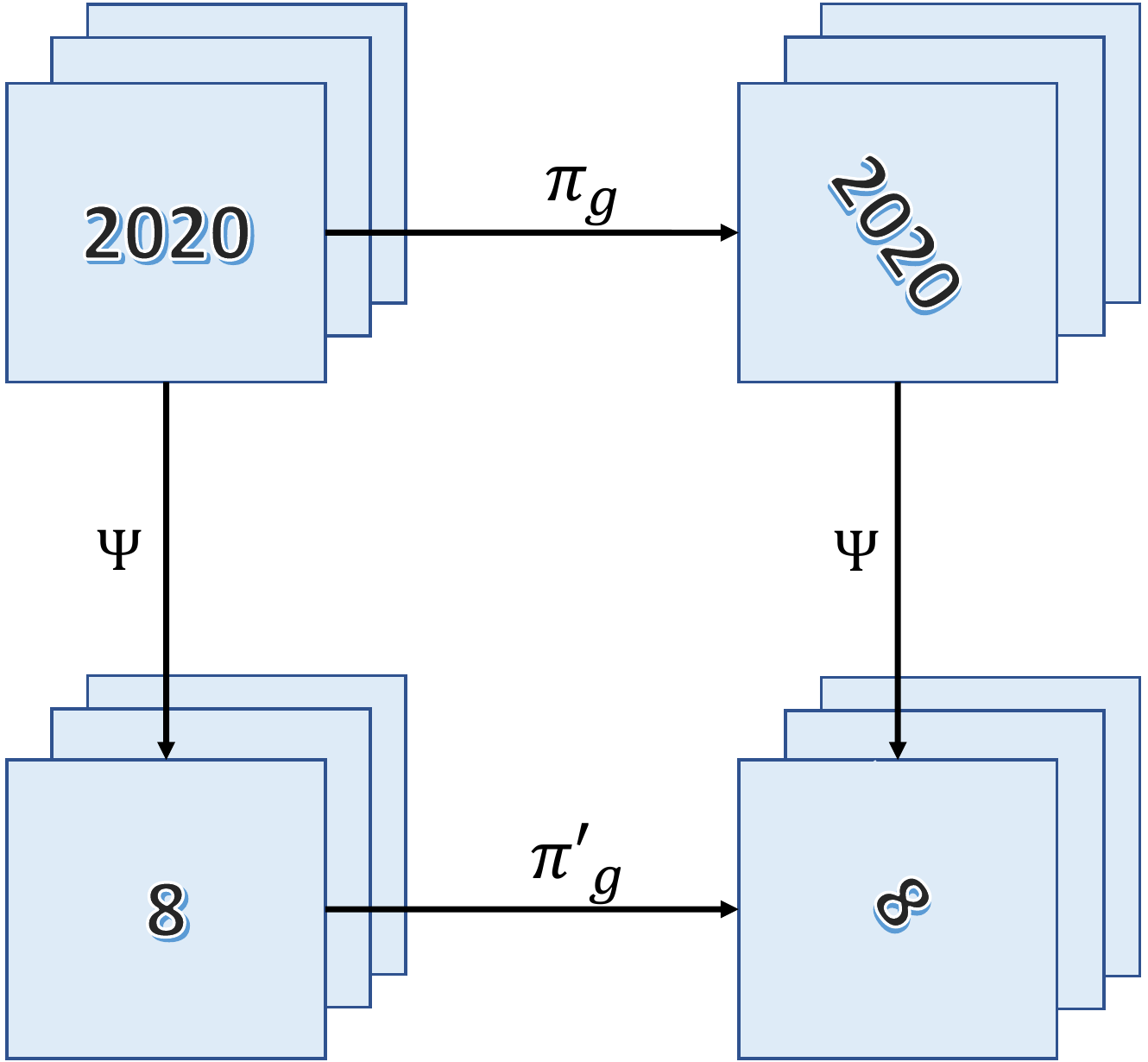} 
	\caption{The transformation $g$ can be preserved by the mapping $\Psi$.}
	\label{equivariance}
\end{figure}

\textbf{Equivariance}\quad Equivariance measures how the outputs of a mapping transform in a predictable way with the transformation of the inputs. Here, we formulate it in detail. Let $\Psi$ be a mapping from the input feature space to the output feature space and $G$ is a group. A group equivariant $\Psi$ satisfies that
\begin{align*}
\forall g \in G,\quad \Psi[\pi_g[f]]=\pi^{\prime}_g[\Psi[f]],
\end{align*}
where $f$ can be any input feature map in the input feature space, and $\pi _g$ and $\pi^{\prime}_g$ denote how the transformation $g$ acts on input features and output features, respectively. 

That is, transforming an input $f$ by a transformation $g$ (forming $\pi_g[f]$) and then passing it through the mapping $\Psi$ should give the same result as first mapping $f$ through $\Psi$ and then transforming the representation. The schema of equivariance is shown in Figure \ref{equivariance}. It is easy to see that if each layer of a network is equivariant, the equivariance can be preserved by the network.

\subsection{Group Equivariant Differential Operators}
We refer to $H(u_1,u_2,\cdots,u_n;\bm{\beta})$ as a polynomial of $n$ variables parameterized by $\bm{\beta}$. $\frac{\partial }{\partial x_i}$ denotes the derivative with respect to the $i$th coordinate of $x$. Obviously, as a polynomial of PDOs $\left\{\frac{\partial} {\partial x_i}\right\}_{i=1}^n$, $H(\frac{\partial}{\partial x_1},\frac{\partial}{\partial x_2},\dots,\frac{\partial}{\partial x_n};\bm{\beta})$ is a linear combination of PDOs parameterized by $\bm{\beta}$. For example, if $H(u_1,u_2;\bm{\beta})=\beta_1u_1+\beta_2 u_1u_2$, then $H(\frac{\partial }{\partial x_1},\frac{\partial }{\partial x_2};\bm{\beta})=\beta_1\frac{\partial}{\partial x_1}+\beta_2\frac{\partial^2}{\partial x_1\partial x_2}$. 

\subsubsection{Under Orthogonal Transformation}
We transform these PDOs with orthogonal matrices, and define the following differential operator:
\begin{align}
\chi^{(A)}=H\left(\frac{\partial}{\partial x_1^{(A)}},\frac{\partial}{\partial x_2^{(A)}}\dots,\frac{\partial}{\partial x_n^{(A)} };\bm{\beta}\right),
\label{aa}
\end{align}
where 
\begin{align}
\begin{bmatrix}
\frac{\partial}{\partial x_1^ {(A)}}\\
\frac{\partial}{\partial x_2^ {(A)}}\\
\vdots\\
\frac{\partial}{\partial x_n^{(A)}}
\end{bmatrix}
=\quad A^{-1} \quad
\begin{bmatrix}
\frac{\partial}{\partial x_1}\\
\frac{\partial}{\partial x_2}\\
\vdots\\
\frac{\partial}{\partial x_n}
\end{bmatrix},
\label{21}
\end{align}
and $A$ is an orthogonal matrix. As a compact format, we can also rewrite (\ref{21}) as 
\begin{equation}
\nabla^{(A)} = A^{-1}\nabla \label{compact},
\end{equation}
where $\nabla=[\frac{\partial}{\partial x_1},\frac{\partial}{\partial x_2},\cdots,\frac{\partial}{\partial x_n}]^T$, which is a gradient operator. Particularly, the canonical operator $\chi^{(I)}=H(\frac{\partial}{\partial x_1},\frac{\partial}{\partial x_2},\cdots,\frac{\partial}{\partial x_n};\bm{\beta})$. From another point of view, the transformation on PDOs can also be viewed as that we transform the coordinate frame according to $A$, and then conduct differential operators on the new coordinate frame (see Figure \ref{rotated}). Particularly, PDOs can be viewed as steerable filters in the sense of \cite{helor1996canonical}, because the transformed versions of PDOs can be expressed as linear combinations of PDOs.

\begin{figure}
	\centering
	\includegraphics[width=0.9\columnwidth]{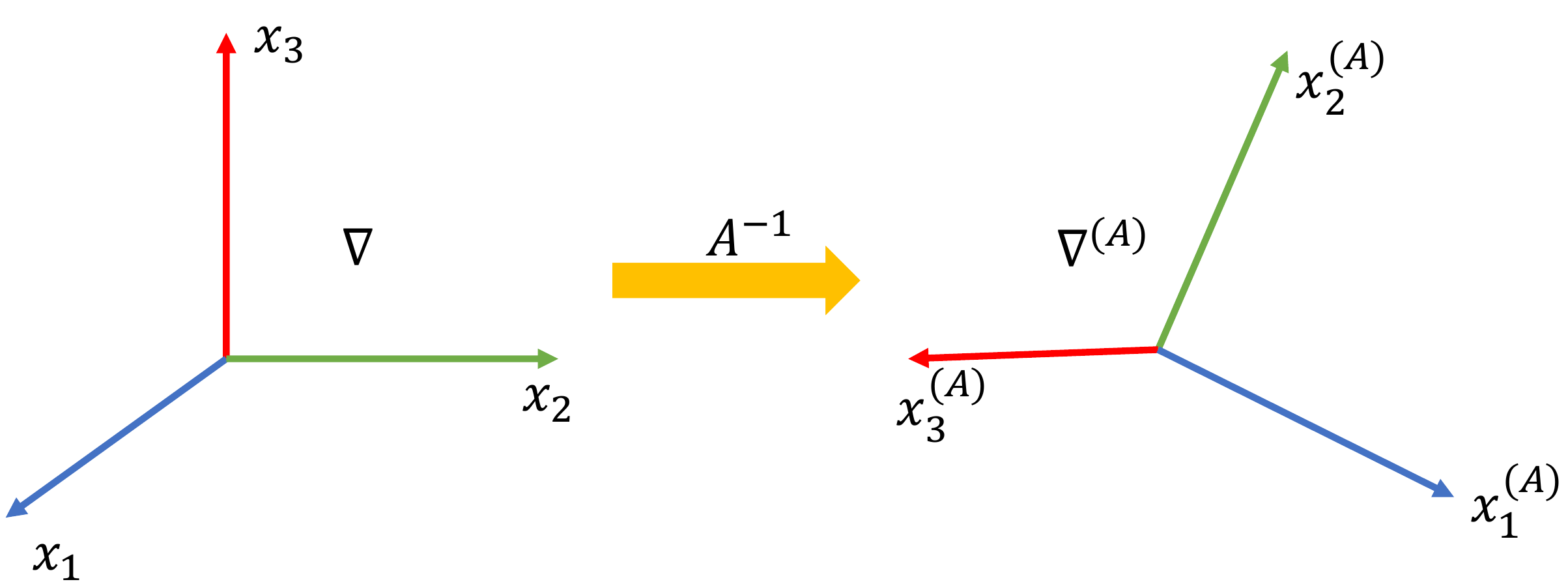} 
	\caption{Transformation over coordinate frame.}
	\label{rotated}
\end{figure}

Next, we employ $\chi^{(A)}$'s to define two differential operators $\Psi$ and $\Phi$. To be specific, we use $\Psi$ to deal with inputs, which maps an input $r\in C^\infty(\mathbb{R}^n)$ to a feature map defined on $E(n)$: $\forall (x,A)\in E(n)$,
\begin{align}
\forall (x,A)\in E(n),\quad \Psi [r](x,A)=\chi^{(A)}[r](x).
\label{1}
\end{align}
Then we use $\Phi$ to deal with the resulting feature maps, which maps one feature map $e\in C^\infty(E(n))$ to another feature map defined on $E(n)$: \\
$\forall (x,A)\in E(n)$,
\begin{align}
\Phi [e](x,A)=\int_{O(n)} \chi^{(A)}_{B}\,\,[e](x,AB) d\nu(B)\label{Phi},
\end{align}
where $B$ is an orthogonal matrix and $\nu$ is a measure on $O(n)$. As for $\chi_B^{(A)}$, we use the subscript $B$ to distinguish the differential operators parameterized by different $\bm{\beta}_B$'s. The $e$ on the right hand side should be viewed as a function defined on $\mathbb{R}^n$ indexed by $AB$ when the operator $\chi^{(A)}_B$ acts on it.

We now show that the above two operators are equivariant under orthogonal transformations and describe how the outputs transform w.r.t. the transformations of inputs.

\begin{theorem}
	If $r \in C^{\infty}(\mathbb{R}^n), e \in C^{\infty}(E(n))$ and $\widetilde{A}\in O(n)$, the following rules are satisfied:
	\begin{align}
	\Psi \left[\pi^{R}_{\widetilde A}[r]\right]=&\pi^{E}_{\widetilde A}\left[\Psi [r]\right],\label{e1}\\
	\Phi \left[\pi^{E}_{\widetilde A}[e]\right] =& \pi^{E}_{\widetilde A}\left[\Phi [e]\right],\label{e2}
	\end{align}
	where $\pi^{R}_{\widetilde A},\pi^{E}_{\widetilde A},\Psi$ and $\Phi$ are defined in (\ref{T1}), (\ref{T2}), (\ref{1}) and (\ref{Phi}), respectively.
	\label{theorem1}
\end{theorem}

\begin{proof}
	To prove (\ref{e1}), we need to prove that $\forall x\in \mathbb{R}^n,A\in O(n)$,
	\begin{align}
	\chi^{(A)}\left[\pi_{\widetilde A}^R[r]\right](x) &= \pi_{\widetilde A}^E\left[\chi^{(A)}[r](x)\right]\notag\\
	& = \chi^{(\tilde A^{-1}A)}[r](\widetilde A^{-1}x).\label{detail}
	\end{align}
	We first show that
	\small{
		\begin{align*}
		\nabla^{(A)} \left[\pi^{R}_{\widetilde A}[r]\right](x)
		& = (A^{-1}\nabla) \left[\pi^{R}_{\widetilde A}[r]\right](x)\\
		&=(A^{-1}\nabla) \left[r( \widetilde A^{-1}x)\right]\\
		&=A^{-1}\widetilde A\nabla [r] (\widetilde{A}^{-1}x)\\
		&=(\widetilde{A}^{-1}A)^{-1}\nabla [r] (\widetilde{A}^{-1}x)\\
		&=\nabla^{(\widetilde{A}^{-1}A)}[r](\widetilde{A}^{-1}x).
		\end{align*}}
	The derivation from the third line to the fourth line is due to the orthogonality of $\widetilde{A}$. Thus for any element $x_i$ in $x$, we have
	\begin{equation*}
	\frac{\partial}{\partial x_i^{(A)}} \left[\pi^{R}_{\widetilde A}[r]\right](x)= \frac{\partial}{\partial x_i^{(\tilde A^{-1}A)}}[r] (\tilde A^{-1}x).
	\end{equation*}
	Furthermore,
	\begin{align*}
	&\nabla^{(A)}\left[\frac{\partial}{\partial x_i^{(A)}} \left[\pi^{R}_{\widetilde A}[r]\right]\right](x)\\
	=&A^{-1}\nabla\left[\frac{\partial}{\partial x_i^{(\widetilde{A}^{-1}A)}} [r](\widetilde{A}^{-1}x)\right]\\
	=&(\widetilde{A}^{-1}A)^{-1}\nabla\left[\frac{\partial}{\partial x_i^{(\widetilde{A}^{-1}A)}} [r]\right](\widetilde{A}^{-1}x)\\
	=&\nabla^{(\widetilde{A}^{-1}A)}\left[\frac{\partial}{\partial x_i^{(\widetilde{A}^{-1}A)}} [r]\right](\widetilde{A}^{-1}x).
	\end{align*}
	Then we have that for any elements $x_i$ and $x_j$ in $x$,
	\begin{equation*}
	\frac{\partial}{\partial x_i^{(A)}}\frac{\partial}{\partial x_j^{(A)}} \left[\pi^{R}_{\widetilde A}[r]\right](x)= \frac{\partial}{\partial x_i^{(\tilde A^{-1}A)}}\frac{\partial}{\partial x_j^{(\tilde A^{-1}A)}}[r] (\tilde A^{-1}x).
	\end{equation*}
	In this way, it is easy to prove that (\ref{detail}) is satisfied for all the differential operator terms in $\chi^{(\cdot)}$. Finally, as $\chi^{(\cdot)}$ is a linear combination of above terms, (\ref{detail}) is satisfied. Easily, (\ref{e1}) is satisfied.
	
	As for (\ref{e2}), similarly, $\forall x\in \mathbb{R}^n,A\in O(n)$,
	\begin{align*}
	\Phi \left[\pi^{E}_{\widetilde{A}}[e]\right](x,A) &= \Phi \left[e(\widetilde{A}^{-1}x,\widetilde{A}^{-1}A)\right] \\
	&= \int_{O(n)} \chi_B^{(A)} \left[e(\widetilde{A}^{-1}x,\widetilde{A}^{-1}AB)\right]d\nu (B)\\
	&= \int_{O(n)} \chi_B^{(A)} \left[\pi^{R}_{\widetilde{A}} [e](x,\widetilde{A}^{-1}AB)\right]d\nu (B)\\
	&= \int_{O(n)} \chi_B^{(\widetilde A^{-1}A)} [e](\widetilde A^{-1}x,\widetilde{A}^{-1}AB)  d\nu (B)\\
	& = \pi^{E}_{\widetilde A} \left[\int_{O(n)} \chi_B^{(A)}  [e](x,AB)d\nu(B)\right]\\
	& = \pi^{E}_{\widetilde{A}} [\Psi [e]] (x,A).
	\end{align*}
	The derivation from the third line to the fourth line is due to (\ref{detail}). So (\ref{e2}) is satisfied.$\hfill\blacksquare$  
\end{proof}

Furthermore, as differential operators are naturally translation-equivariant, it is easy to verify that $\Psi$ and $\Phi$ are also equivariant over $E(n)$. Consequently, according to the working spaces, we set a $\Psi$ as the first layer,
followed by multiple $\Phi$'s, inserted by pointwise nonlinearities, e.g., ReLUs, that do not disturb the equivariance. Finally, we can get a system where equivariance can be preserved across multiple layers.

\subsubsection{Under Subgroup of Orthogonal Transformation \label{subgroup}}
The above theorem can be easily extended to subgroups of $E(n)$. Here we consider a subgroup $\tilde{E}(n)$ with the form $\mathbb{R}^n \rtimes S$, where $S$ is a subgroup of $O(n)$. 
Similarly, we denote the smooth feature map defined on $\tilde{E}(n)$ as $\tilde{e}$ and the function space as $C^{\infty}(\tilde{E}(n))$.

The definition of the differential operator $\Psi^S$ is the similar with (\ref{1}):
\begin{align}
\forall (x,A)\in \tilde{E}(n),\quad \Psi^S [r](x,A)=\chi^{(A)}[r](x),
\label{5}
\end{align}
where the only difference is that $A \in S$. If $S$ is a discrete group, the differential operator $\Phi^S$ is:
\begin{align}
\forall (x,A)\in \tilde{E}(n),\quad\Phi^S [\tilde e] (x,A)=\sum\limits_{B\in S} \chi^{(A)}_{B} [\tilde e](x,AB),
\label{6}
\end{align}
where $A \in S$. Following (\ref{T1}) and (\ref{T2}), we can define $\pi^{R}_{\widetilde A}$ and $\pi^{\tilde{E}}_{\widetilde A}$, where $\tilde A \in S$. We can get the similar result:
\begin{align}
\Psi^S \left[\pi^{R}_{\widetilde A}[r]\right]=&\pi^{\tilde{E}}_{\widetilde A}\left[\Psi^S[r]\right],\\
\Phi^S \left[\pi^{\tilde{E}}_{\widetilde A}[\tilde e]\right]=&\pi^{\tilde E}_{\widetilde A}\left[\Phi^S [\tilde e]\right]. \label{8}
\end{align}
Easily, they are also equivariant w.r.t. $\tilde E(n)$.

\section{PDO-eConvs \label{section4}}
In this section, we apply our theory to $2$D digital images, and derive approximately equivariant convolutions in the discrete domain. As they are designed using PDOs, we refer to them as PDO-eConvs. To begin with, we show how to apply PDOs on discrete images and feature maps with convolutional filters, respectively.

\subsection{Differential Operators Acting on Discrete Features}

We can view discrete digital images as samples from smooth functions defined on the 2D plane. Formally, we assume that an image data $\bm{I}\in \mathbb{R}^{n\times n}$ represents a two-dimensional grid function obtained by discretizing a smooth function $r:[0,1]\times [0,1]:\to \mathbb{R}$ at the cell-centers of a regular grid with $n\times n$ cells and a mesh size $h=1/n$, i.e., for $i,j=1,2,\dots,n,$
\begin{align*}
\bm{I}_{i,j}=r(x_i,y_j), 
\end{align*}
where $x_i=(i-\frac{1}{2})h$ and $y_j=(j-\frac{1}{2})h$. 

Accordingly, intermediate feature maps in CNNs are multi-channel matrices. Similarly, 
it can be seen as the  discretizations of continuous functions defined on $\tilde E$, where $\tilde E =\mathbb{R}^2\rtimes S$ and $S$ is a subgroup of $O(2)$. Formally, a feature map $\bm{F}$ represents a three-dimensional grid function sampled from a smooth function $e:[0,1]^2\times S:\to \mathbb{R}$.
For $i,j=1,2,\dots n$,
\begin{align}
\bm{F}^k_{i,j}=e(x_i,y_j,k),
\end{align}
where $x_i=(i-\frac{1}{2})h,y_j=(j-\frac{1}{2})h$ and $k\in S$ which represents its channel index. Here, for ease of presentation, we only consider that inputs and intermediate feature maps are all single-valued functions, and the theory can be easily extended to multi-valued functions.

With the understanding that features are sampled from continuous functions, we can implement differential operations on features. Particularly, we use convolutions to approximate differential operations, which have been widely used in image processing. For example, the operator$\frac{\partial }{\partial x}$ acting on images and feature maps can be approximated by the following $3\times3$ convolutional filter with quadratic precision:
\begin{align}
\frac{\partial }{\partial x}[r](x_i,y_j)&=\left(\frac{1}{2h}\begin{bmatrix}
0&0&0\\
-1&0&1\\
0&0&0
\end{bmatrix}\ast \bm{I}\right)_{i,j}+O(h^2),\notag\\
\frac{\partial }{\partial x}[e](x_i,y_j,k)&=\left(\frac{1}{2h}\begin{bmatrix}
0&0&0\\
-1&0&1\\
0&0&0
\end{bmatrix}\ast \bm{F}^k\right)_{i,j}+O(h^2),\notag
\label{9}
\end{align}
where $\ast$ denotes the convolution operation.

\subsection{From Group Equivariant Differential Operators to PDO-eConvs}
Firstly, we choose the polynomial $H$ from the connection between differential operators and convolutions. Ruthotto \& Haber \yrcite{ruthotto2018deep} showed that we can relate a $3\times3$ convolutional filter to a differential operator, $\mathcal{D}$, which is a linear combination of $9$ linearly independent PDOs\footnote{For ease of presentation, we denote the identity operator as $\partial_0$, and view it as a  special PDO.}.
\begin{align}
\mathcal{D}=&\beta_1\partial_0+\beta_2 \partial_{x}+\beta_3 \partial_{y}+\beta_4 \partial_{xx}+\beta_5\partial_{xy}\\
&+\beta_6\partial_{yy}+\beta_7\partial_{xxy}+\beta_8\partial_{xyy}+\beta_9\partial_{xxyy}.\notag
\label{combination}
\end{align}
In addition, we observe that all differential operators in (\ref{combination}) can be approximated using $3\times3$ convolutional filters (see Supplementary Material 1.1) with quadratic precision. It is to say that we can always approximate the differential operators defined in (\ref{combination}) using a $3\times 3$ filter with quadratic precision. For this reason, we choose 
\begin{align}
H(u,v;\bm{\beta})=&\beta_1+\beta_2u+\beta_3v+\beta_4u^2+\beta_5uv\\
&+\beta_6v^2+\beta_7u^2v+\beta_8uv^2+\beta_9u^2v^2.\notag
\end{align}
In this way, $\mathcal{D}$ equals $\chi^{(I)}$, which is also the canonical differential operator of $\chi^{(A)}$'s, indexed by the identity matrix. Using the transformation in (\ref{21}), we can calculate all the expressions of $\chi^{(A)}$'s easily. Particularly, these transformed differential operators share the same parameters $\bm{\beta}$, indicating greater parameter efficiency.

In computation, we observe that some new partial derivatives, e.g., $\partial_{xxx},\partial_{xxxx}$, may occur in some $\chi^{(A)}$'s, where $A\in S$. Fortunately, the orders of these new partial derivatives are all below five, and we can use the filters with the size of $5\times 5$ (see Supplementary Material 1.2) to approximate them with quadratic precision. 

Now we investigate the group we use. According to (\ref{Phi}) and (\ref{6}), if $S$ is a continuous group, we need to conduct integration. However, for the computation issue, it seems impossible to consider all the orthogonal transformations in $O(2)$.
So we consider $S$ to be a discrete subgroup of $O(2)$. Still, our theory is satisfied for feature maps defined on $\tilde{E}$ (see Section \ref{subgroup}). 
Particularly, noting that $O(2)$ is generated by reflections and rotations, we set the subgroup $S$ to be generated by reflections and rotations by $2\pi/n$. 
As a result, $\tilde{E}=pnm$. If without reflections, $\tilde{E}=pn$. Discrete groups $pnm$ and $pn$ have been introduced in Section \ref{intro}.

Finally, we discretize the equivariant differential operator $\Psi$ with corresponding convolutional filters. As a result, we can get a new operator, $\tilde{\Psi}$, which is actually a set of convolution operators indexed by $A$:
\begin{equation}
\forall A\in S, \quad \tilde{\chi}^{(A)}=\sum_{i\in \Gamma} C^{(A)}_i \tilde u_i, \label{20}
\end{equation}
where $\Gamma$ indexes all the filters we use, $C_i^{(A)}$ are derived by substituting (\ref{21}) into (\ref{aa}) and $\tilde u_i$ is the convolutional filter related to the PDO $\partial_i$ (e.g., $\tilde u_0$ and $\tilde u_{xy}$ are related to $\partial_0$ and $\partial_{xy}$, respectively), then 
\begin{equation}
(\widetilde{\Psi}\ast \bm{I})^A=\tilde{\chi}^{(A)}\ast \bm{I}.
\end{equation}

Similarly, we can get a new convolution operator $\tilde{\Phi}$ by discretizing (\ref{6}). Without ambiguity, we also use $*$ to denote the corresponding convolution operation. To be specific,
\begin{align}
\forall A\in S, \quad \left(\tilde{\Phi}\ast \bm{F}\right)^A=\sum\limits_{k\in S}\tilde{\chi}^{(A)}_k \ast \bm{F}^{Ak},\label{Phiconv}
\end{align}
where $Ak$ is a group product on the group $S$, which respresents the channel index of $\bm{F}$, and $\bm{F}^{Ak}\in \mathbb{R}^{n\times n}$.

We refer to $\tilde{\Psi}$ and $\tilde{\Phi}$ as PDO-eConvs, because they are equivariant convolutions based on PDOs. Following \cite{cohen2016group}, we replace all the conventional convolutions in an existing CNN with our PDO-eConvs, and get the corresponding group equivariant CNN w.r.t. $\tilde{E}$.

\begin{figure}[t]
	\centering
	\includegraphics[width=0.5\columnwidth]{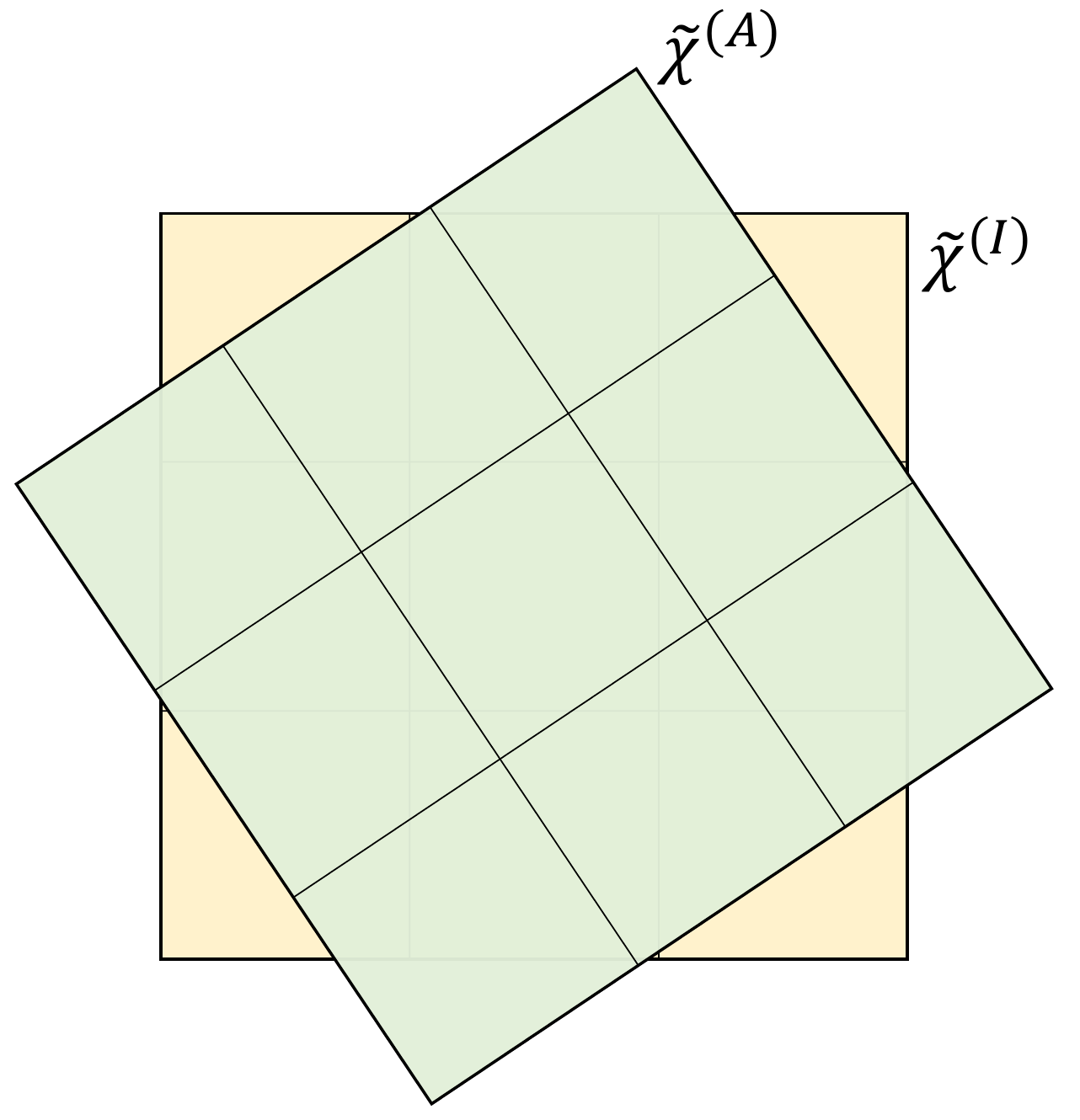} 
	\caption{The canonical convolutional filter $\tilde\chi^{(I)}$ and its rotated version $\tilde\chi^{(A)}$.}
	\label{filters}
\end{figure}
Let us have a more detailed look at (\ref{20}). Some convolutional filters like $u_{xxxx}$ are of size $5\times 5$, thus for some $A\in \tilde{E}$, $\tilde{\chi}^{(A)}$ is also of size $5\times 5$, while the canonical convolutional filter $\tilde{\chi}^{(I)}$ is of size $3\times 3$. We can explain the phenomenon in this way. By definition, the differential operator $\chi^{(A)}$ is transformed from $\chi^{(I)}$. Intuitively, we can also view the convolutional filter $\tilde{\chi}^{(A)}$ as a transformed version of $\tilde{\chi}^{(I)}$. We assume the transformation to be the rotation. As shown in Figure \ref{filters}, $\tilde \chi^{(A)}$ is a rotated version of $\tilde\chi^{(I)}$, which overflows the original $3\times 3$ area. So it makes sense to use a larger filter to represent some transformed filters.
That $5\times5$ is sufficient is because the rotated $3\times3$ mask can always be covered by a $5\times5$ square, noting that $5 \geq 3\sqrt{2}$.

\subsection{Approximation Error of Equivariance}
When we discretize the differential operators $\Psi$ and $\Phi$, errors occur, leading to equivariance disturbance. Nonetheless, we can still achieve approximate equivariance. Here, we analyze the approximation error of our PDO-eConvs. 

%

\begin{theorem}
	$\forall \tilde{A}\in S$,
	\begin{align}
	\tilde{\Psi} \ast\pi^R_{\widetilde A}[\bm{I}]&=\pi^{\tilde E}_{\widetilde A}\left[\tilde{\Psi}\ast\bm{I}\right]+O(h^2),\label{appro}\\
	\tilde{\Phi}\ast \pi^{\tilde E}_{\widetilde A}[\bm{F}]&=\pi^{\tilde E}_{\widetilde A}\left[\tilde{\Phi}\ast\bm{F}\right]+O(h^2),
	\label{14}
	\end{align}
	where transformations such as rotations or mirror reflections acting on images are defined as $(\pi^R_{\widetilde A}[\bm{I}])_{i,j}=(\pi^R_{\widetilde A}[r])(x_i,y_j)$ and transformations acting on feature maps are $(\pi^{\tilde E}_{\widetilde A}[\bm{F}])^k_{i,j}=(\pi^{\tilde E}_{\widetilde A}[e])(x_i,y_j,k)$.
\end{theorem}

\begin{proof}
	$\forall A\in S$, the operator ${\chi}^{(A)}$ is a linear combination of differential operators and $\tilde{\chi}^{(A)}$ is a combination
	of corresponding convolution operators. Hence if $r$ is a smooth function,
	\begin{align}
	\chi^{(A)} [r](x_i,y_j) =&\left(\tilde{\chi}^{(A)} \ast \bm{I}\right)_{i,j}+O(h^2),\notag\\
	\chi^{(A)} \left[\pi^R_{\widetilde A}[r]\right](x_i,y_j) =&\left(\tilde{\chi}^{(A)} \ast \pi^R_{\widetilde A}[\bm{I}]\right)_{i,j}+O(h^2),\notag
	\end{align}
	i.e.,
	\begin{align}
	\Psi [r](x_i,y_j,A) =&\left(\tilde{\Psi} \ast \bm{I}\right)^A_{i,j}+O(h^2),\notag\\
	\Psi \left[\pi^R_{\widetilde A}[r]\right](x_i,y_j,A) =&\left(\tilde{\Psi} \ast \pi^R_{\widetilde A}[\bm{I}]\right)^A_{i,j}+O(h^2).\label{l1}
	\end{align}
	Easily, we have
	\begin{equation}
	\pi^{\tilde E}_{\widetilde A}\left[\Psi [r]\right](x_i,y_j,A) =\left( \pi^{\tilde E}_{\widetilde A}\left[\tilde{\Psi} \ast \bm{I}\right]\right)^A_{i,j}+O(h^2).\label{l2}
	\end{equation}
	From (\ref{e1}) we know that the left hand sides of (\ref{l1}) and (\ref{l2}) equal, hence the right hand sides of the two equation are the same, which results in (\ref{appro}). We can prove (\ref{14}) analogously.$\hfill\blacksquare$  
\end{proof}

\subsection{Weight Initialization Scheme \label{initialize}}
An important practical issue in the training phase is an appropriate initialization of weights.
When the variances of weights are chosen too high or too low, 
the signals propagating through the network are amplified or suppressed exponentially with depth. 
Glorot \& Bengio \yrcite{glorot2010understanding} and He et al. \yrcite{he2015delving} investigated this problem and proposed widely used initialization schemes. However, our filters are not parameterized in a pixel basis but as linear combinations of several PDOs,
thus the above-mentioned initialization schemes cannot directly be adopted for our PDO-eConvs.

To be specific, we consider the canonical filter $\tilde \chi^{(I)}$ in each PDO-eConv, and initialize it with He's initialization scheme \cite{he2015delving}. Then we initialize the parameters $\bm{\beta}$ of the PDO-eConv by solving the linear equation
\begin{align}
\tilde \chi^{(I)}=&\beta_1 \tilde u_0+\beta_2 \tilde u_{x}+\beta_3 \tilde u_{y}+\beta_4 \tilde u_{xx}+\beta_5 \tilde u_{xy}\\
&+\beta_6 \tilde u_{yy}+\beta_7 \tilde u_{xxy}+\beta_8 \tilde u_{xyy}+\beta_9 \tilde u_{xxyy}.\notag
\end{align}
with the initialized $\tilde \chi^{(I)}$. In this way, the canonical filter is initialized with He's initialization scheme. Since other filters are obtained by transforming the canonical filters, they also have appropriate variances. We initialize each $\tilde \Psi_k$ in (\ref{Phiconv}) in the same way. We use this method to initialize all the PDO-eConvs in experiments and all the experiments are implemented using Tensorflow.

\section{Experiments}
\subsection{Rotated MNIST}

The most commonly used dataset for validating rotation-equivariant algorithms is MNIST-rot-12k \cite{larochelle2007empirical}. It contains the handwritten digits of the classical MNIST, rotated by a random angle from $0$ to $2\pi$ (full angle). This dataset contains 12,000 training images and 50,000 test images, respectively. We randomly select 2,000 training images as a validation set. We choose the model with the lowest validation error during training. For preprocessing, we normalize the images using the channel means and standard deviations.

\textbf{Without Data Augmentation}\quad Firstly, we evaluate the performance of PDO-eConvs on MNIST-rot-12k without data augmentation via the CNN architecture used in \cite{cohen2016group}. It contains $6$ layers of $3\times 3$ convolutions, $20$ channels in each layer, ReLU functions, batch normalization \cite{ioffe2015batch}, and max pooling after layer $2$. 

We consider the group $p8$ and replace each convolution by a $p8$-convolution, divided the number of filters by $\sqrt{8}$, in order to keep the numbers of parameters nearly the same. Thus we use $7$ filters on each layer. Particularly, batch normalization should be implemented with a single scale and a single bias per PDO-eConv map to preserve equivariance. 

The model is trained using the Adam algorithm \cite{kingma2014adam} with a weight decay of $0.01$. We use the weight initialization method introduced in Section \ref{initialize} for PDO-eConvs and Xavier initialization \cite{glorot2010understanding} for the fully connected layer. We train using batch size $128$ for $200$ epochs. The initial learning rate is set to $0.001$ and is divided by $10$ at $50\%$ and $75\%$ of the total number of training epochs. We set the dropout rate as $0.2$.
\begin{table}[t]
	\caption{Error rates on MNIST-rot-12k without data augmentation.}\smallskip
	\centering
	\linespread{1.2}\selectfont
	\resizebox{0.8\columnwidth}{!}{
		\smallskip\begin{tabular}{lcc}
			\hline
			Network & Test Error ($\%$) & params\\
			\hline
			ScatNet-2 \cite{bruna2013invariant} & 7.48 & -\\
			PCANet-2 \cite{chan2015pcanet} & 7.37 & -\\
			TIRBM \cite{sohn2012learning} & 4.2 & -\\
			\hline
			ORN-8 (ORNAlign) \cite{zhou2017oriented} & 2.25 & 0.53M\\
			TI-Pooling \cite{laptev2016ti} & 2.2 & 13.3M\\
			\hline
			CNN & 5.03 & 22k\\
			G-CNN \cite{cohen2016group}& 2.28 & 25k\\
			\textbf{PDO-eConv (ours)} & \textbf{1.87} & 26k\\
			\hline
		\end{tabular}
	}
	\label{rot}
\end{table}

As shown in Table \ref{rot}, with comparable numbers of parameters, our proposed PDO-eConv achieves $1.87\%$ test error, outperforming conventional CNN ($5.03\%$) and G-CNN ($2.28\%$), which is equivariant on group $p4$. This is mainly because that our model is rotation-equivariant w.r.t. smaller rotation angles, which brings in better generalization. ORN-$8$ also deals with an $8$-fold rotational symmetry and adopts an extra strategy, ORNAlign, to refine feature maps. Compared with ORN-8 (ORNAlign), our method still results in lower test error, using far fewer numbers of parameters (26k vs. 0.53M). TI-Pooling is a representative model of transformation-invariant CNNs, which use parallel siamese architectures. Compared with it, PDO-eConv performs better ($1.87\%$ vs. $2.2\%$) using far fewer parameters (26k vs. 13.3M) and has much lower computational complexity.

\begin{table}[b]
	\caption{Competitive results on MNIST-rot-12k.}\smallskip
	\centering
	\linespread{1.2}\selectfont
	\resizebox{0.8\columnwidth}{!}{
		\smallskip\begin{tabular}{lcc}
			\hline
			Method & Test Error ($\%$) \\
			\hline
			H-Net \cite{worrall2017harmonic} & 1.69\\
			OR-TIPooling \cite{zhou2017oriented} & 1.54\\
			RotEqNet \cite{marcos2017rotation} & 1.09 \\
			PTN-CNN \cite{esteves2018polar} & 0.89\\
			E2CNN \cite{cesa2019general} & 0.716 \\
			SFCNN \cite{weiler2018learning} & 0.714\\		
			\hline
			\textbf{PDO-eConv (ours)} & \textbf{0.709} \\
			\hline
		\end{tabular}
	}
	\label{sota}
\end{table}

\textbf{Competitive Result with Data Augmentation}\quad We compare the performance of our PDO-eConv with some more competitive models, using data augmentation and a larger model with $7$ layers. These layers have 16, 16, 32, 32, 32, 64 and 64 output channels, respectively. We use spatial pooling and orientation pooling after the final PDO-eConv layer, in order to get rotation-invariant features. Following \cite{weiler2018learning}, we augment the dataset with continuous rotations during training time. This model is trained using stochastic gradient descent (SGD) and a Nesterov momentum \cite{sutskever2013importance} of $0.9$ without dampening. We train this model for $300$ epochs, starting with a learning rate of $10^{-2}$ and reducing it gradually to $10^{-5}$.

As shown in Table \ref{sota}, E2CNN and SFCNN achieve $0.716\%$ and $0.714\%$ test error on rotated MNIST, respectively. Compared with SFCNN, our method achieves a comparable result, $0.709\%$ test error, using only $10\%$ parameters. To be specific, our method uses 0.65M parameters, while SFCNN needs 6.5M parameters. Also, SFCNN used a much larger architecture and larger kernel sizes ($7\times 7$ and $9\times 9$), which relate to a much larger computational cost. E2CNN replicates the architecture used in SFCNN, so it also relates to a huge computational cost.

\subsection{Natural Image Classfication}
Although most objects in natural scene images are up-right, rotations could exist in small scales. Besides, equivariance to a transformation group brings in more parameter sharing, which may improve the parameter efficiency. Here we evaluate the performance of our PDO-eConvs on two common natural image datasets, CIFAR-10 (C10) and CIFAR-100 (C100) \cite{krizhevsky2009learning}, respectively.

The two CIFAR datasets consist of colored natural images with $32\times 32$ pixels. C10 consists of images drawn from 10 classes and C100 from 100. The training and the test sets contain 50,000 and 10,000 images, respectively. We randomly select 5,000 training images as a validation set. We choose the model with the lowest validation error during training. We adopt a standard data augmentation scheme (mirroring/shifting) \cite{lee2015deeply} that is widely used for these two datasets. For preprocessing, we normalize the images using the channel means and standard deviations.

To evaluate our method, we take ResNet \cite{he2016identity} as the basic model, which consists of an initial convolution layer, followed by three stages of $2n$ convolution layers using $k_i$ filters at stage $i$, followed by a final classification layer ($6n + 2$ layers in total). We replace all convolution layers of ResNets by our PDO-eConvs and implement batch normalization with a single scale and a single bias per PDO-eConv map. Also, we scale the number of filters to keep the numbers of parameters approximately the same. All the models are trained using SGD and a Nesterov momentum \cite{sutskever2013importance} of $0.9$ without dampening. We train using batch size $128$ for $300$ epochs, weight decay of $0.001$. The initial learning rate is set to $0.1$ and is divided by $10$ at $50\%$ and $75\%$ of the total number of training epochs. Similarly, we use the weight initialization method introduced in Section \ref{initialize} for our PDO-eConvs and Xavier initialization for the fully connected layer. We report the results of our methods in Table \ref{natural images}.

\begin{table}
	\caption{Results on the natural image classification benchmark. In the second column, $G$ is the group where equivariance can be preserved.}\smallskip
	\centering
	\linespread{1.2}\selectfont
	\resizebox{1.0\columnwidth}{!}{ \smallskip
		\centering
		\linespread{1.3}\selectfont
		\begin{tabular}{l|cc|cc|c}
			\hline
			Method & $G$ & Depth& C10 & C100  & params\\
			\hline
			ResNet \cite{he2016identity}  & $\mathbb{Z}^2$ & 26  &11.5 & 31.66 & 0.37M\\
			HexaConv \cite{hoogeboom2018hexaconv} & $p6$  & 26& 9.98 &-  & 0.34M\\
			& $p6m$ & 26 & 8.64 &-  & 0.34M\\
			PDO-eConv (ours) & $p6$  & 26 & 5.65 & 27.13 & 0.36M \\
			& $p6m$  & 26& 5.38 & 27.00  & 0.37M\\
			\hline
			ResNet & $\mathbb{Z}^2$  & 44& 5.61 & 24.08  & 2.64M\\
			G-CNN \cite{cohen2016group} & $p4m$ & 44 & 4.94 & 23.19  & 2.62M\\
			PDO-eConv (ours) & $p8$  & 44& 3.68 & 20.01  & 2.62M\\
			\hline
			ResNet & $\mathbb{Z}^2$ & 1001 & 4.92 & 22.71 & 10.3M\\
			Wide ResNet \cite{zagoruyko2016wide} & $\mathbb{Z}^2$ & 26  & 4.00 & 19.25 & 36.5M\\
			G-CNN \cite{cohen2016group} & $p4m$ & 26 & 4.17 & -  & 7.2M\\
			\textbf{PDO-eConv (ours)} & $p8$ & 26 & \textbf{3.50} & \textbf{18.40}  & 4.6M\\
			\hline
		\end{tabular}
		\label{natural images}
	}
	\label{table1}
\end{table}

Following HexaConv, we use our PDO-eConvs to establish models that are equivariant to group $p6$ ($p6m$), where $n=4$ and $k_i=6,13,26$ ($k_i=6,9,18$). Using comparable numbers of parameters, our methods perform significantly better than HexaConv ($5.38\%$ vs. $8.64\%$ on C10). In addition, HexaConvs require extra memory to store hexagonal images while our PDO-eConvs do not need so. 

We evaluate PDO-eConvs using ResNet-44, where $n=7$ and $k_i = 11,23,45$. Compared with G-CNNs, our PDO-eConvs achieve significantly better performance using comparable numbers of parameters ($3.68\%$ vs. $4.94\%$ on C10, and $20.01\%$ vs. $23.19\%$ on C100). When evaluated on ResNet-26, where  $n=4,k_i=20,40,80$, PDO-eConv results in $3.50\%$ test error, much better than $4.17\%$ resulted from G-CNN, yet using much fewer parameters (4.6M vs. 7.2M). This is mainly because that PDO-eConvs can deal with an $8$-fold rotational symmetry, which exploit more rotational symmetries compared with G-CNN.

Finally, we compare our models with deeper ResNets (ResNet-1001) and wider ResNets (Wide ResNet). As shown in Table \ref{natural images}, PDO-eConvs perform betterr ($3.50\%$ vs. $4.00\%$  in C10 and $18.40\%$ vs. $19.25\%$ in C100) using only $12.6\%$ parameters (4.6M vs. 36.5M). Particularly, PDO-eConvs can also be viewed as introducing a weight sharing scheme across channels, and the results indicate that our method can not only save parameters, but also improve the performance remarkably.

\section{Conclusion}

We utilize PDOs to design a system which is exactly equivariant to a much more general continuous group, the $n$-dimension Euclidean group. We use numerical schemes to implement these PDOs and derive approximately equivariant convolutions, PDO-eConvs. Particularly, we provide an error analysis and show that the approximation error is of the quadratic order. Extensive experiments verify the effectiveness of our method.

In this work, we only conduct experiments on 2D images. Actually, our theory can deal with the data with any dimension. We will explore more possibilities in the future.

\section*{Acknowledgements}
This work was supported by the National Key Research and Development Program of China under grant 2018AAA0100205. Z. Lin is supported by NSF China (grant no.s 61625301 and 61731018), Major Scientific Research Project of Zhejiang Lab (grant no.s 2019KB0AC01 and 2019KB0AB02), Beijing Academy of Artificial Intelligence, and Qualcomm.

\bibliography{example_paper}
\bibliographystyle{icml2020}

\appendix
\section*{Supplementary Material}
\section{Numerical Schemes of Partial Differential Operators}
\subsection{Filters of Size $3\times 3$}

\begin{align*}
	& \widetilde u_0=\left[
	\begin{array}{p{0.7cm}<{\centering} p{0.7cm}<{\centering} p{0.7cm}<{\centering}}
		0&0&0\\
		0&1&0\\
		0&0&0
	\end{array}
	\right]\\
	&\widetilde u_x=\frac{1}{h}\left[
	\begin{array}{p{0.7cm}<{\centering} p{0.7cm}<{\centering} p{0.7cm}<{\centering}}
		0&0&0\\
		-1/2&0& 1/2\\
		0&0&0
	\end{array}
	\right]\\
	&\widetilde u_y=\frac{1}{h}\left[
	\begin{array}{p{0.7cm}<{\centering} p{0.7cm}<{\centering} p{0.7cm}<{\centering}}
		\quad 0\quad&1/2&  0\quad\\
		\quad 0\quad&0& 0\quad\\
		\quad 0\quad&-1/2&0\quad
	\end{array}
	\right]\\
	&\widetilde u_{xx}=\frac{1}{h^2}\left[
	\begin{array}{p{0.7cm}<{\centering} p{0.7cm}<{\centering} p{0.7cm}<{\centering}}
		0&0&0\\
		1&-2&1\\
		0&0&0
	\end{array}
	\right]\\
	&\widetilde u_{xy}=\frac{1}{h^2}\left[
	\begin{array}{p{0.7cm}<{\centering} p{0.7cm}<{\centering} p{0.7cm}<{\centering}}
		-1/4 &0& 1/4\\
		0&0&0\\
		1/4&0&-1/4
	\end{array}
	\right]\\
	&\widetilde u_{yy}=\frac{1}{h^2}\left[
	\begin{array}{p{0.7cm}<{\centering} p{0.7cm}<{\centering} p{0.7cm}<{\centering}}
		0&1&0\\
		0&-2&0\\
		0&1&0
	\end{array}
	\right]\\
	&\widetilde u_{xxy}=\frac{1}{h^3}\left[
	\begin{array}{p{0.7cm}<{\centering} p{0.7cm}<{\centering} p{0.7cm}<{\centering}}
		1/2&-1&1/2\\
		0& 0 &0\\
		-1/2 &1& -1/2
	\end{array}
	\right]\\
	&\widetilde u_{xyy}=\frac{1}{h^3}\left[
	\begin{array}{p{0.7cm}<{\centering} p{0.7cm}<{\centering} p{0.7cm}<{\centering}}
		-1/2&0&1/2\\
		1 & 0 & -1\\
		-1/2 &0 & 1/2
	\end{array}
	\right]\\
	&\widetilde u_{xxyy}=\frac{1}{h^4}\left[
	\begin{array}{p{0.7cm}<{\centering} p{0.7cm}<{\centering} p{0.7cm}<{\centering}}
		1&-2&1\\
		-2 & 4 & -2\\
		1 &-2 & 1
	\end{array}
	\right]
\end{align*}

\subsection{Filters of Size $5\times 5$}
\begin{align*}
	&\widetilde u_{xxx}=\frac{1}{h^3}\left[
	\begin{array}{p{0.7cm}<{\centering} p{0.7cm}<{\centering} p{0.7cm}<{\centering}p{0.7cm}<{\centering}p{0.7cm}<{\centering}}
		0&0&0&0&0\\
		0&0&0&0&0\\
		-1/2&1&0&-1&1/2\\
		0&0&0&0&0\\
		0&0&0&0&0\\
	\end{array}
	\right]\\
	& \widetilde u_{yyy}=\frac{1}{h^3}\left[
	\begin{array}{p{0.7cm}<{\centering} p{0.7cm}<{\centering} p{0.7cm}<{\centering}p{0.7cm}<{\centering}p{0.7cm}<{\centering}}
		0&0&1/2&0&0\\
		0&0&-1&0&0\\
		0&0&0&0&0\\
		0&0&1&0&0\\
		0&0&-1/2&0&0\\
	\end{array}
	\right]
\end{align*}

\begin{align*}
	& \widetilde u_{xxxx}=\frac{1}{h^4}\left[
	\begin{array}{p{0.7cm}<{\centering} p{0.7cm}<{\centering} p{0.7cm}<{\centering}p{0.7cm}<{\centering}p{0.7cm}<{\centering}}
		0&0&0&0&0\\
		0&0&0&0&0\\
		1&-4&6&-4&1\\
		0&0&0&0&0\\
		0&0&0&0&0\\
	\end{array}
	\right]\\
	& \widetilde u_{xxxy}=\frac{1}{h^4}\left[
	\begin{array}{p{0.7cm}<{\centering} p{0.7cm}<{\centering} p{0.7cm}<{\centering}p{0.7cm}<{\centering}p{0.7cm}<{\centering}}
		0&0&0&0&0\\
		-1/4&1/2&0&-1/2&1/4\\
		0&0&0&0&0\\
		1/4&-1/2&0&1/2&-1/4\\
		0&0&0&0&0\\
	\end{array}
	\right]\\
	& \widetilde u_{xyyy}=\frac{1}{h^4}\left[
	\begin{array}{p{0.7cm}<{\centering} p{0.7cm}<{\centering} p{0.7cm}<{\centering}p{0.7cm}<{\centering}p{0.7cm}<{\centering}}
		0&-1/4&0&1/4&0\\
		0&1/2&0&-1/2&0\\
		0&0&0&0&0\\
		0&-1/2&0&1/2&0\\
		0&1/4&0&-1/4&0\\
	\end{array}
	\right]\\
	& \widetilde u_{yyyy}=\frac{1}{h^4}\left[
	\begin{array}{p{0.7cm}<{\centering} p{0.7cm}<{\centering} p{0.7cm}<{\centering}p{0.7cm}<{\centering}p{0.7cm}<{\centering}}
		0&0&1&0&0\\
		0&0&-4&0&0\\
		0&0&6&0&0\\
		0&0&-4&0&0\\
		0&0&1&0&0\\
	\end{array}
	\right]
\end{align*}

\end{document}